\documentclass{article}

\usepackage{microtype}
\usepackage{graphicx}
\usepackage{subfigure}
\usepackage{booktabs} 

\usepackage[accepted]{icml2021}

\usepackage{amsmath, amssymb, amsthm}
\usepackage{graphicx, graphbox}
\usepackage{dsfont}

\usepackage{pgf, tikz}
\usetikzlibrary{arrows, automata, positioning}

\newcommand{\R}{\mathbb{R}}
\newcommand{\E}{\mathbb{E}}
\newcommand{\1}{\mathds{1}}

\usepackage{accents}

\usepackage{pgfplots}
\pgfplotsset{compat=1.14}

\newtheorem{theorem}{Theorem}
\newtheorem{definition}{Definition}

\newtheorem{lemma}{Lemma}
\newtheorem{assumption}{Assumption}

\usepackage[utf8]{inputenc} 
\usepackage[T1]{fontenc}    
\usepackage{url}            
\usepackage{booktabs}       
\usepackage{amsfonts}       
\usepackage{nicefrac}       
\usepackage{microtype}      

\usepackage{wrapfig}
\usepackage{thmtools, thm-restate}

\icmltitlerunning{Bandit Overfitting in Offline Policy Learning}

\begin{document}




\twocolumn[
\icmltitle{Offline Contextual Bandits with Overparameterized Models}



\icmlsetsymbol{equal}{*}

\begin{icmlauthorlist}
\icmlauthor{David Brandfonbrener}{nyu}
\icmlauthor{William F. Whitney}{nyu}
\icmlauthor{Rajesh Ranganath}{nyu}
\icmlauthor{Joan Bruna}{nyu}
\end{icmlauthorlist}

\icmlaffiliation{nyu}{Courant Institute of Mathematical Sciences, New York University, New York, New York, USA}

\icmlcorrespondingauthor{David Brandfonbrener}{david.brandfonbrener@nyu.edu}

\icmlkeywords{bandits, offline bandits, batch bandits, reinforcement learning, offline reinforcement learning, batch reinforcement learning}

\vskip 0.3in
]



\printAffiliationsAndNotice{}  

\begin{abstract}
    Recent results in supervised learning suggest that while overparameterized models have the capacity to overfit, they in fact generalize quite well.
    We ask whether the same phenomenon occurs for offline contextual bandits.
    Our results are mixed. 
    Value-based algorithms benefit from the same generalization behavior as overparameterized supervised learning, but policy-based algorithms do not. 
    We show that this discrepancy is due to the \emph{action-stability} of their objectives.
    An objective is action-stable if there exists a prediction (action-value vector or action distribution) which is optimal no matter which action is observed.
    While value-based objectives are action-stable, policy-based objectives are unstable.
    We formally prove upper bounds on the regret of overparameterized value-based learning and lower bounds on the regret for policy-based algorithms.
    In our experiments with large neural networks, this gap between action-stable value-based objectives and unstable policy-based objectives leads to significant performance differences.

    
    
\end{abstract}

\section{Introduction}


The offline contextual bandit problem can be used to model decision making from logged data in domains as diverse as recommender systems \citep{li2010contextual, bottou2013counterfactual}, healthcare \citep{Prasad2017ARL, Raghu2017DeepRL}, and robotics \citep{pinto2016supersizing}.
Prior work on the problem has primarily focused on underparameterized models with finite and small VC dimensions. This work has come from the bandit literature \citep{strehl2010learning, swaminathan2015counterfactual, swaminathan2015self}, the reinforcement learning literature \citep{munos2008finite, chen2019information}, and the causal inference literature \citep{bottou2013counterfactual, athey2017efficient, kallus2018balanced, zhou2018offline}.

In contrast, the best performance in modern supervised learning is often achieved by massively overparameterized models that are capable of fitting random labels \cite{zhang2016understanding}. Use of such large models renders vacuous the bounds that require a small model class. But, the massive capacity of popular neural network models is now often viewed as a feature rather than a bug. Large models reduce approximation error and allow for easier optimization \cite{du2018gradient} while still being able to generalize in regression and classification problems \cite{belkin2018overfitting, belkin2019does}. 
In this paper, we investigate whether the strong performance of overparameterized models in supervised learning translates to the offline contextual bandit setting.
The main prior work that considers this setup is \cite{joachims2018deep}, which we discuss in detail in Section \ref{sec:related}.

To formalize the differences between the supervised learning and contextual bandit settings, we introduce a novel regret decomposition.
This decomposition shares the approximation and estimation terms from classic work in supervised learning \cite{vapnik1982, bottou2008tradeoffs}, but adds a term for ``bandit'' error which captures the excess risk due to only receiving partial feedback.


We use this framework to address the question: can we use overparameterized models for offline contextual bandits? Or is the bandit error a fundamental problem when we use large models? 
We find mixed results. 
Value-based algorithms benefit from the same generalization behavior as overparameterized supervised learning, but policy-based algorithms do not. 
We show that this difference is explained by a property of their objectives called \emph{action-stability}.
An objective is action-stable if there exists a single prediction which is simultaneously optimal for any observed action (where a ``prediction'' is a vector of state-action values for a value-based objective or an action distribution for a policy-based objective).
Action-stable objectives perform well when combined with overparameterized models since the random actions taken by the behavior policy do not change the optimal prediction.
However, interpolating an unstable objective results in learning a different function for every sample of actions, even though the true optimal policy remains unchanged.

On the theory side, we prove that overparameterized value-based algorithms are action stable and have small bandit error via reduction to overparameterized regression.
Meanwhile we prove that policy-based algorithms are not action-stable which allows us to prove lower bounds on the ``in-sample'' regret and lower bounds on the regret for simple nonparametric models.

Empirically, we demonstrate the gap in both action stability and bandit error between policy-based and value-based algorithms when using large neural network models on synthetic and image-based datasets.

In summary, our main contributions are:
\begin{itemize}
    \item We introduce the concept of bandit error, which separates contextual bandits from supervised learning.
    \item We introduce action-stability and show that a lack of action-stability causes bandit error.
    \item We show a gap between policy-based and value-based algorithms based on action-stability and bandit error both in theory and experiments. 
\end{itemize}



\section{Setup}\label{sec:setup}
\subsection{Offline contextual bandit problem}

First we will define the contextual bandit problem \citep{langford2008epoch}. Let the context space $ \mathcal{X} $ be infinite and the action space $ \mathcal{A}$ be finite with $ |\mathcal{A}| = K < \infty$. 
At each round, a context $ x \in \mathcal{X}$ and a full feedback reward vector $ r \in [r_{\mathrm{min}},r_{\mathrm{max}}]^K$ are drawn from a joint distribution $ \mathcal{D}$. 
Note that $ r $ can depend on $ x $ since they are jointly distributed. 
A policy $\pi: \mathcal{X} \to \mathcal{P}(\mathcal{A})$ maps contexts to distributions over actions. An action $ a $ is sampled according to $ \pi(a|x)$ and the reward is $ r(a)$, the component of the vector $ r $ corresponding to $ a$. 
We use ``bandit feedback'' to refer to only observing $ r(a)$.
This contrasts with the ``full feedback'' problem
where at each round the full vector of rewards $ r$ is revealed, independent of the action. 

In the offline setting there is a finite dataset of $ N $ rounds with a fixed behavior policy $ \beta$.  Then we denote the dataset as $ S = \{x_i, r_i, a_i, p_i\}_{i=1}^N$ where $ p_i $ is the observed propensity $ p_i = \beta(a_i|x_i)$. The tuples in the datasets lie in $ \mathcal{X} \times [r_{min},r_{max}]^K \times \mathcal{A} \times [0,1]$ and are drawn i.i.d from the joint distribution induced by $ \mathcal{D}$ and $ \beta$.
From $ S $ we define the datasets $ S_B$ for bandit feedback and $ S_F$ for full feedback:
\begin{align*}
    S_B = \{(x_i, r_i(a_i), a_i, p_i)\}_{i=1}^N, \quad S_F = \{(x_i, r_i)\}_{i=1}^N.
\end{align*}
Note that we are assuming access to the behavior probabilities $ p_i = \beta(a_i|x_i)$, so the issues that we raise do not have to do with estimating propensities.
We will further make the following assumption about the behavior.
\begin{assumption}[Strict positivity]\label{ass:positivity}
We have strict positivity of $ \tau $ if $ \beta(a|x) \geq \tau > 0$ for all $ a, x$. Thus, in any dataset we will have $ p_i = \beta(a_i|x_i) \geq \tau > 0$.
\end{assumption}
There is important work that focuses learning without strict positivity by making algorithmic modifications like clipping \citep{bottou2013counterfactual, strehl2010learning, swaminathan2015counterfactual} and behavior constraints \citep{fujimoto2018off, laroche2019safe}. However, these issues are orthogonal to the main contribution of our paper, so we focus on the setting with strict positivity.


The goal of an offline contextual bandit algorithm is to take in a dataset and produce a policy $ \pi$ so as to maximize the value $ V(\pi)$ defined as
\begin{align*}
    V(\pi) := \E_{x, r \sim \mathcal{D}}\E_{a \sim \pi(\cdot|x)} [r(a)].
\end{align*}
We will use $ \pi^*$ to denote the deterministic policy that maximizes $ V$. 
Finally, define the $ Q $ function at a particular context, action pair as
\begin{align*}
    Q(x,a) := \E_{r|x}[r(a)].
\end{align*}

\subsection{Model classes}

The novelty of our setting comes from the use of overparameterized model classes that are capable of interpolating the training objective. 
To define this more formally, all of the algorithms we consider take a model class of either policies $ \Pi$ or $ Q$ functions $ \mathcal{Q}$ and optimize some objective over the data with respect to the model class.
Following the empirical work of \cite{zhang2016understanding} and theoretical work of \cite{belkin2018overfitting} we will call a model class ``overparameterized'' or ``interpolating'' if the model class contains a model that exactly optimizes the training objective. Formally, if we have data $ \{x_i\}_{i=1}^N$ and a pointwise loss function $ \ell(x, y)$, then a model class $ \Pi$ can interpolate the data if 
\begin{align*}
    \inf_{\pi\in \Pi} \sum_{i=1}^N \ell(x_i, \pi(x_i)) = \sum_{i=1}^N \inf_{y} \ell(x_i, y).
\end{align*}
This contrasts with traditional statistical learning settings where we assume that the model class is finite or has low complexity as measured by something like VC dimension \cite{strehl2010learning, swaminathan2015counterfactual}.

\subsection{Algorithms}\label{sec:basic-algs}

Now that we have defined the problem setting, we can define the algorithms that we will analyze. This is not meant to be a comprehensive account of all algorithms, but a broad picture of the ``vanilla'' versions of the main families of algorithms. Since we are focusing on statistical issues we do not consider how the objectives are optimized.

\paragraph{Supervised learning with full feedback.} In a full feedback problem, empirical value maximization (the analog to standard empirical risk minimization) is defined by maximizing the empirical value $ \hat V_F$: 
\begin{align}
    \hat V_F(\pi; S_F) &:= \frac{1}{N}\sum_{i=1}^N \langle r_i, \pi(\cdot|x_i)\rangle\\  \pi_F &:= \arg\max_{\pi\in \Pi} \hat V_F(\pi;S_F).
\end{align}

\paragraph{Policy-based learning.}
Importance weighted or ``inverse propensity weighted'' policy optimization directly optimizes the policy to maximize an estimate of its value. Since we only observe the rewards of the behavior policy, we use importance weighting to get an unbiased value estimate to maximize. Explicitly: 
\begin{align}
    \hat V_B(\pi ; S_B) &:= \frac{1}{N}\sum_{i=1}^N r_i(a_i) \frac{\pi(a_i|x_i)}{p_i} \\ \pi_B &:= \arg \max_{\pi \in \Pi} \hat V_B(\pi;S_B). \label{eq:pi}
\end{align}
Note that this is the ``vanilla'' version of the policy-based algorithm and modifications like regularizers, baselines/control variates, clipped importance weights, and self-normalized importance weights have been proposed \citep{bottou2013counterfactual, joachims2018deep, strehl2010learning, swaminathan2015counterfactual, swaminathan2015self}. 
For our purposes considering this vanilla version is sufficient since as we show in Section \ref{sec:stable}, any objective that takes the form $ \pi(a_i|x_i) f(x_i, a_i, r_i, p_i)$ at each datapoint will have the same sort of problem with action-stability.

It is important to note that with underparameterized model classes, this algorithm is guaranteed to return nearly the best policy in the class. Explicilty, \citet{strehl2010learning} prove that for a finite policy class $ \Pi$, with high probability the regret of the learned policy $ \pi_B$ is bounded as $ O(\frac{1}{\tau} \sqrt{\frac{\log |\Pi|}{N}})$. This is elaborated in Appendix \ref{app:small}. However, these guarantees no longer hold in our overparameterized setting.


\paragraph{Value-based learning.}
Another simple algorithm is to first learn the $ Q $ function and then use a greedy policy with respect to this estimated $ Q$ function. Explicitly:
\begin{align}\label{eqn:hatQ}
    \hat Q_{S_B} &:= \arg\min_{f \in \mathcal{Q}} \sum_{i=1}^N (f(x_i, a_i) - r_i(a_i))^2 \\ \pi_{\hat Q_{S_B}}(a|x) &:= \1 \left[a = \arg\max_{a'} \hat Q_{S_B}(x, a')\right].
\end{align}
This algorithm is sometimes called the ``direct method'' \citep{dudik2011doubly}. 
The RL literature also often defines a class of model-based algorithms, but in the contextual bandit problem there are no state transitions so model-based algorithms are equivalent to value-based algorithms.

This algorithm also has a guarantee with small model classes. Explicilty, \citet{chen2019information} prove that for a finite \emph{and well specified} model class $ \mathcal{Q}$, with high probability the regret of the learned policy $ \pi_{\hat Q_{S_B}}$ is bounded as $ O(\frac{1}{\sqrt{\tau}} \sqrt{\frac{\log |\mathcal{Q}|}{N}})$. This is elaborated in Appendix \ref{app:small}. Again, these guarantees no longer hold in our overparameterized setting.

\paragraph{Doubly robust policy optimization.}

The class of doubly robust algorithms \citep{dudik2011doubly} does not fall cleanly into the value-based or policy-based bins since it requires first learning a value function and using that to optimize a policy. However, with overparamterized models, doubly robust learning becomes exactly equivalent to our vanilla value-based algorithm unless we use crossfitting since the estimated Q values will coincide with the rewards. We prove this formally in Appendix \ref{app:dr} where we also show some issues that the doubly robust policy objective can have with overparameterized models and highly stochastic rewards. For our purposes, we will only consider the policy-based and value-based approaches since the doubly robust approach collapses to the value-based approach with overparameterized models.

\section{Bandit error}\label{sec:decomp}

In supervised learning, the standard decomposition of the excess risk separates the approximation and estimation error  \citep{bottou2008tradeoffs}. The approximation error is due to the limited function class and the estimation error is due to minimizing the empirical risk rather than the true risk. Since the full feedback policy learning problem is equivalent to supervised learning, the same decomposition applies. Formally, consider a full feedback algorithm $ \mathcal{A}_F$ which takes the dataset $ S_F $ and produces a policy $ \pi_F$. Then
\begin{align*}
    &\underbrace{\E_{S}[V(\pi^*) - V(\pi_F)]}_{\textrm{regret}} = \underbrace{V(\pi^*) - \sup_{\pi\in \Pi}V(\pi)}_{\textrm{approximation\ error}} \\ &\qquad+ \underbrace{\E_{S}[ \sup_{\pi\in \Pi}V(\pi) - V(\pi_F)]}_{\textrm{estimation\ error}}.
\end{align*}
We can instead consider a bandit feedback algorithm $ \mathcal{A}_B$ which takes the dataset $ S_B $ and produces a policy $ \pi_B$.
To extend the above decomposition to the bandit problem we add a new term, the bandit error, that results from having access to $ S_B$ rather than $ S_F$. Now we have:
\begin{align*}
    &\underbrace{\E_{S}[V(\pi^*) - V(\pi_B)]}_{\textrm{regret}} = \underbrace{V(\pi^*) - \sup_{\pi\in \Pi}V(\pi)}_{\textrm{approximation\ error}} \\&\quad+ \underbrace{\E_{S}[ \sup_{\pi\in \Pi}V(\pi) - V(\pi_F)]}_{\textrm{estimation\ error}} + \underbrace{\E_{S}[V(\pi_F) - V(\pi_B)]}_{\textrm{bandit\ error}}.
\end{align*}

\paragraph{Disentangling sources of error.} 
The approximation error is the same quantity that we encounter in the supervised learning problem, measuring how well our function class can do.
The estimation error measures the error due to overfitting on finite contexts and noisy rewards.
The bandit error accounts for the error due to only observing the actions chosen by the behavior policy. 
This is not quite analogous to overfitting to noise in the rewards since stochasticity in the actions is actually required to have the coverage of context-action pairs needed to learn a policy.
While the standard approximation-estimation decomposition could be directly extended to the bandit problem, our approximation-estimation-bandit decomposition is more conceptually useful since it disentangles these two types of error.

\paragraph{Can bandit error be negative?}
Usually, we think about an error decomposition as a sum of positive terms. 
This is not necessarily the case with our decomposition, but we view this as a feature rather than a bug.
Intuitively, the bandit error term captures the contribution of the actions selected by the behavior policy. 
If the behavior policy is nearly optimal and the rewards are highly stochastic, there may be more signal in the actions selected by the behavior policy than the observed rewards. 
Thus overfitting the actions chosen by behavior policy can sometimes be beneficial, causing the bandit error to be negative. 
The two terms disentangle the approximation error (due to reward noise) from bandit error (due to behavior actions).

\section{Action-stable objective functions}\label{sec:stable}


Consider a simple thought experiment.
We collect a contextual bandit dataset $S_B$ from a two-action environment using a uniformly random behavior policy.
Then we construct a second dataset $\widetilde S_B$ by evaluating the outcome of taking the opposite action at each observed context.
Since nothing about the environment has changed, we know that the optimal policy remains the same.
Therefore we desire the following property from a bandit objective: there exists a single model which is optimal (with respect to that objective) on both $S_B$ and $\widetilde S_B$.
We say that such an objective is \emph{action-stable} because it has an optimal policy which is stable to re-sampling of the actions in the dataset.


More formally, we define action stability pointwise at a datapoint $ z = (x,r,p)$ where $ r\in [r_{\min}, r_{\max}]^K$ and $ p \in \Delta^K$ is the behavior probability vector in the $K$-dimensional simplex (recall that $ K $ is the number of the actions). Let $ z(a)$ denote the datapoint when action $ a $ is sampled so that $ z(a) = (x, r(a), p(a), a)$. The objectives for both policy and value-based algorithms decompose into sums over the data of some loss $ \ell(z(a), \pi(a|x))$ or $ \ell(z(a), Q(x,a))$.

Note that the output space of a policy is the simplex so that $ \pi(\cdot|x) \in \Delta^K$, while the output of a Q function\footnote{When using neural networks Q is usually implemented as a function of $ x $ with $ K $ outputs \cite{mnih2015human}} is $ Q(x, \cdot) \in \R^K$. 
To allow for this difference in our definition, we will define a generic $ K$-dimensional output space $ \mathcal{Y}^K$ and its corresponding restriction to one dimension as $ \mathcal{Y}$. 
So for a policy-based algorithm $ \mathcal{Y}^K = \Delta^K$ and $ \mathcal{Y} = [0,1]$, while for a value-based algorithm $ \mathcal{Y}^K = \R^K$ and $ \mathcal{Y} = \R$. Now we can define action-stability. 

\begin{definition}[Action-stable objective]
An objective function $ \ell$ is action-stable at a datapoint $ z$ if there exists $ y^* \in \mathcal{Y}^K$ such that for all $ a \in \mathcal{A}$:
\begin{align*}
    \ell(z(a), y^*(a)) = \min_{y \in \mathcal{Y}} \ell(z(a), y).
\end{align*}
\end{definition}

If an objective is not action-stable, a function which minimizes that objective exactly at every datapoint $(x, r(a), p(a), a)$ does \emph{not} minimize it for a different choice of $a$.
As a direct consequence, interpolating an unstable objective results in learning a different function for every sample of actions, even though the true optimal policy remains unchanged.

We find that policy-based objectives are not action-stable, while value-based objectives are.
In the next section we will use the instability of policy-based objectives to show that policy-based algorithms exhibit large bandit error when used with overparameterized models.
Our stability results are stated in the following two Lemmas, whose proofs can be found in Appendix \ref{app:stable}.




\begin{restatable}[Value-based stability]{lemma}{vbstable}\label{lem:vb-stable}
Value-based objectives are action stable since we can let $ y^* = r$ and this minimizes the square loss at every action.
\end{restatable}

\begin{restatable}[Policy-based instability]{lemma}{pbstable}\label{lem:pb-stable}
All policy-based objectives which take the form $ \ell(z(a), \pi(a|x)) = f(z(a)) \pi(a|x)$ are not action-stable at $ z$ unless $ f(z(a)) > 0$ for exactly one action $ a$. 
\end{restatable}

These Lemmas tell us that the stochasticity of the behavior policy can cause instability for policy-based objectives. 
This is worrisome since one would hope that more stochastic behavior policies give us more information about all the actions and should thus yield better policies. Indeed, this is the case for value-based algorithms as we will see in the next section. But for policy-based algorithms, stochastic behavior can itself be a cause of overfitting due to the instability of the objective function.

\paragraph{Stabilizing policy-based algorithms.} To avoid this problem in a policy-based algorithm, the sign of the function $f(z(a))$ must indicate the optimal action. This essentially requires having access to a \emph{baseline} function $ b(s) $ that separates the optimal action from all the others so that $ r(a) > b(s)$ if and only if $ a $ is the optimal action. And then $  f(z(a)) = \frac{r(a) - b(s)}{\beta(a|s)}$ yields an action-stable algorithm. This is in general as difficult as learning the full value function $ Q$. One notable special case is when the bandit problem is induced by an underlying classification problem, so that only one action has reward 1 and all others have 0. In this case, any constant baseline between 0 and 1 will lead to action stability. This case has often been considered in the literature, e.g. by \citet{joachims2018deep} as we discuss in Section \ref{sec:related}.

Now that we have built up an understanding of the problem we can prove some formal results that show how value-based algorithms more effectively leverage overparameterized models by being action-stable.

\section{Regret bounds}

Recall that as explained in Section \ref{sec:setup}, both policy-based and value-based algorithms have regret guarantees when we use small model classes \cite{strehl2010learning, chen2019information}. But, when we move to the overparameterized setting, this is no longer the case. In this section we prove regret upper bounds for value-based learning by using recent results in overparameterized regression. Then we prove lower bounds on the regret of policy-based algorithms due to their action-instability.

\subsection{Value-based learning}

In this section we show that value-based algorithms can provably compete with the optimal policy. 
The key insight is to reduce the problem to regression and then leverage the guarantees on overparameterized regression from the supervised learning literature. 
This is formalized by the following theorem.


\begin{restatable}[Reduction to regression]{theorem}{reduction}\label{thm:reduction}
By Assumption \ref{ass:positivity} we have $ \beta(a|x) \geq \tau$ for all $ x,a$. Then with $ \hat Q_{S_B}$ as defined in (\ref{eqn:hatQ}) we have  
\begin{align*}
    V(\pi^*) - V(\pi_{\hat Q_{S_B}}) \leq \frac{2}{\sqrt{\tau}} \sqrt{\mathop{\mathbb{E}}_{x, a \sim \beta}[(Q(x,a) - \hat Q_{S_B}(x,a))^2]}.
\end{align*}
\end{restatable}

A proof can be found in Appendix \ref{app:value}. Similar results are presented as intermediate results in \citet{chen2019information, munos2008finite}.
The implication of this result that we want to emphasize is that any generalization guarantees for overparameterized regression immediately become guarantees for value-based learning in offline contextual bandits. Essentially, Theorem \ref{thm:reduction} gives us a regret bound in any problem where overparameterized regression works.
The following results from the overparameterized regression literature demonstrate a few of these guarantees, which all require some sort of regularity assumption on the true $Q$ function to bound the regression error: 




\begin{itemize}
    \item The results of \cite{bartlett2020benign} give finite sample rates for overparameterized linear regression by the minimum norm interpolator depending on the covariance matrix of the data and assuming that the true function is realizable.
    \item The results of \cite{belkin2019does} imply that under smoothness assumptions on $ Q$, a particular singular kernel will interpolate the data and have optimal non-parametric rates. After applying our reduction, the rates are no longer optimal for the policy learning problem due to the square root.
    \item The results of \cite{bach2017breaking} show how choosing the minimum norm infinite width neural network in a particular function space can yield adaptive finite sample guarantees for many types of underlying structure in the $ Q $ function.
    \item The results of \cite{cover1968estimation} imply the consistency of a one nearest neighbor regressor when the rewards are noiseless and $ Q $ is piecewise continuous. This will contrast nicely with Theorem \ref{thm:nn} below. 
\end{itemize}
Each of these guarantees implies a corresponding corollary to Theorem \ref{thm:reduction} resulting in a regret bound for that particular combination of model and assumptions on $ Q$.

\subsection{Policy-based learning}

Now we will show how the policy-based learning algorithms can provably fail because they lack action-stability. We will do this in a few ways. First, we will show that on the contexts in the dataset an action-unstable algorithm must suffer regret. This means that we cannot even learn the optimal policy at the contexts seen during training. Then to deal with generalization beyond the dataset we will prove a regret lower bound for a specific overparameterized model, namely one nearest neighbor. Finally, we discuss a conjecture that such lower bounds can be extended to richer model classes like neural networks.

Since we are proving lower bounds, making any more simplifying assumptions only makes the bound stronger. As such, all of our problem instances that create the lower bounds have only two actions ($ K = 2$).

\paragraph{Regret on the observed contexts.} Before considering how a policy generalizes off of the data, it is useful to consider what happens at the contexts in the dataset. This is especially true for overparameterized models which can exactly optimize the objective on the dataset. To do this, we will define the value of a policy $ \pi$ on the contexts in a dataset $ S $ (which we will call the ``in-sample'' value) by 
\begin{align}
    V(\pi; S) := \frac{1}{N}\sum_{i=1}^N \E_{r|x_i}\E_{a\sim \pi(\cdot|x_i)}[r(a)].
\end{align}
Then the following Theorem shows that the policy $\pi_B$ learned by the simple policy-based algorithm in Equation (\ref{eq:pi}) must suffer substantial regret on $S$. 

\begin{restatable}[In-sample regret lower bound]{theorem}{vsthm}\label{thm:vs}
Let $ K=2$ and the policy class be overparameterized. 
Define $ \Delta_r(x) = \left|\E_{r|x} [r(1) - r(2)]\right|$ as the absolute expected gap in rewards at $ x$.
Define $p_u(x)$ to be the probability that the policy-based objective is action-unstable at $ x$. Recall that $ \beta(a|x) \geq \tau$ by Assumption \ref{ass:positivity}.
Then
\begin{align*}
    \E_S[V(\pi^*;S) - &V(\pi_B; S)] \geq  \tau \E_{x}\big[ p_u(x) \Delta_r(x)\big].  
\end{align*}
\end{restatable}

The full proof is in Appendix \ref{app:pb}.
This Theorem tells us that as often as the objective is action-unstable, we can suffer substantial regret even \emph{on the contexts in the dataset}. We now offer some brief intuition of the proof. When we have two actions and an algorithm is not action-stable at $ x $, then action chosen by the learned policy $ \pi_B$ at $ x_i$ is directly dependent on the observed action $ a_i$. Since the behavior $ \beta$ will choose each action with probability at least $ \tau$ by Assumption \ref{ass:positivity}, the learned policy $ \pi_B$ must choose the suboptimal action with probability at least $ \tau$ at $ x_i$. This causes unavoidable regret for unstable algorithms, as formally stated in Theorem \ref{thm:vs}. 

Note that Theorem \ref{thm:vs} essentially says that action-unstable algorithms convert noise in the behavior into regret. This is the essential problem with unstable algorithms. Rather than using stochasticity in the behavior policy to get estimates of counterfactual actions, an action-unstable algorithm is sensitive to this stochasticity like it is label noise in supervised learning. 

In Appendix \ref{app:noisy} we present a result that makes this connection between behavior policy noise and action-instability more direct. Specifically we show a reduction that takes any classification problem and turns it into a bandit problem such that optimizing the policy-based objective is equivalent to solving a noisy variant of the classification problem. On the contrary, optimizing the full-feedback objective is equivalent to the noiseless classification problem. 

The limitation of this result is that it only applies in-sample and does not rule out that the model class could leverage its inductive bias to perform well away from the training data.
Next we convert this in-sample regret lower bound into a standard regret lower bound for a particular simple interpolating model class, the nearest neighbor policy.

\paragraph{Regret with generalization: nearest neighbor models.}

The above result shows what happens at the contexts \emph{in the dataset} $ S$. It seems that only pathological combinations of model class and problem instance could perform poorly on $ S$ but recover strong performance off of the data. However, it cannot be ruled out a priori if the model class has strong inductive biases to generalize well. In this section we will show that at least for a very simple overparameterized model class, the generalization of the model does not improve performance. 

The following theorem shows that the simplest interpolating model class, a one nearest neighbor rule, fails to recover the optimal policy even in the limit of infinite data.

\begin{restatable}[Regret lower bound for one nearest neighbor]{theorem}{nn}\label{thm:nn}
Let $ \Delta_r = r_{\mathrm{max}} - r_{\mathrm{min}}$. 
Then there exist problem instances with noiseless rewards where
\begin{align*}
    \limsup_{N\to \infty} \E_S[V(\pi_F) - V(\pi_B)] &=  \frac{\Delta_r}{2},
\end{align*}
but
\begin{align*}
    \limsup_{N\to \infty} \E_S[V(\pi^*) - V(\pi_F)] &= 0. 
\end{align*}
\end{restatable}


The proof is in Appendix \ref{app:pb}.
This result shows that using a nearest neighbor scheme to generalize based on the signal provided by the policy-based objective is not sufficient to learn an optimal policy.
Importantly, note that since the rewards are noiseless, a nearest neighbor policy does recover the optimal policy with full feedback and Theorem \ref{thm:reduction} shows that value-based algorithms also recover the optimal policy in this setting. So, the model class is capable of solving the problem, it is the action-unstable algorithm that is causing irreducible error.

\paragraph{More complicated model classes.}
The above result for nearest neighbor models is illustrative, but does not apply to richer model classes like neural networks. While we were not able to construct such lower bounds, we conjecture that they do exist and hope that future work can prove them. We have two reasons to believe that such lower bounds exist. First, Theorem \ref{thm:vs} is agnostic to model class. For a policy to perform well despite poor performance in-sample would require strong inductive biases from the model class. Proving lower bounds requires ruling out such inductive biases as we have shown for nearest neighbor rules. Second, our empirical results presented in the next section show that policy-based algorithms have action-instability and high bandit error with neural networks. The inductive biases are not enough to overcome the poor in-sample performance.

\section{Experiments}\label{sec:exp}

In this section we experimentally verify that the phenomena described by the theory above extend to practical settings that go slightly beyond the assumptions of the theory.\footnote{Code can be found at \url{https://github.com/davidbrandfonbrener/deep-offline-bandits}.} Specifically we want to verify the following with overparameterized neural network models:
\begin{enumerate}
    \item Policy-based algorithms are action-unstable while value-based algorithms are action-stable.
    \item This causes high bandit error for policy-based algorithms, but not value-based algorithms.
\end{enumerate}

We will break the section into two parts. First we consider a synthetic problem using simple feed-forward neural nets and then we show similar phenomena when the contexts are high-dimensional images and the models are Resnets \cite{he2016deep}.

\subsection{Synthetic data}

First, we will clearly demonstrate action-stability and bandit error in a synthetic problem with a linear reward function.
Specifically, we sample some hidden reward matrix $ \theta$ and then sample contexts and rewards from isotropic Gaussians:
\begin{align*}
    \theta \sim U([0,1]^{K\times d}),\quad x \sim \mathcal{N}(0, I_d), \quad r \sim \mathcal{N}(\theta x, \epsilon I_d).
\end{align*}
Actions are sampled according to a uniform behavior:
\begin{align*}
    a\sim \beta(\cdot|x) = U(\{1,\dots,K\}).
\end{align*}

For these experiments we set $ K=2, d = 10, \epsilon = 0.1$. We take $ N = 100$ training points and sample an independent test set of 500 points. As our models we use MLPs with one hidden layer of width 512. In our experience, the findings are robust to all these hyperparameters of the problem so long as the model is overparameterized. Full details about the training procedure along with learning curves and further results are in Appendix \ref{app:experiments}.

\begin{figure}[h]
    \centering
    \includegraphics[width=0.48\textwidth]{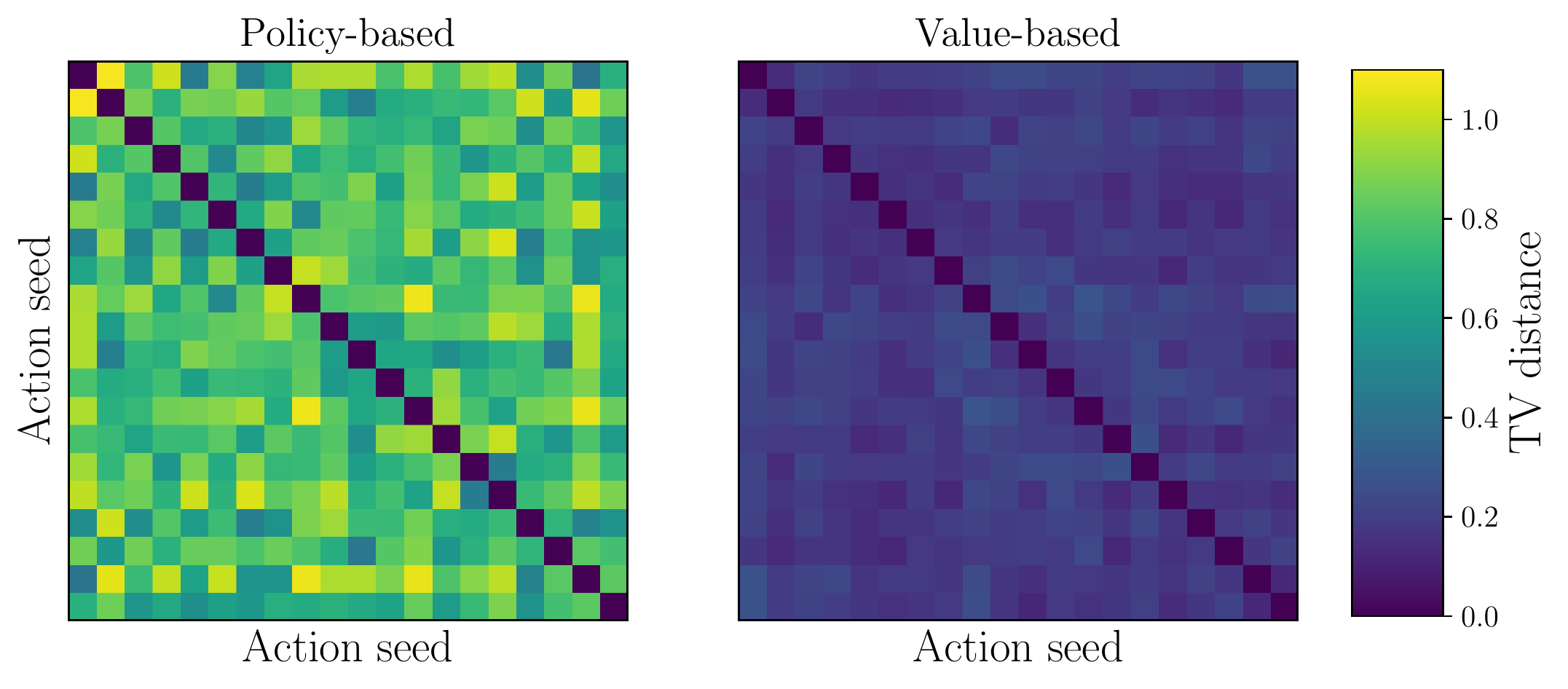}
    \caption{We test action-stability by resampling the actions 20 times for a single dataset of contexts. Each pixel corresponds to the pair of action seeds $ i,j $ and the color shows the TV distance between $ \pi_i(\cdot|x)$ and $ \pi_j(\cdot|x) $ on a held-out test set sampled from the data generating distribution. The policy-based algorithms are highly sensitive to the randomly sampled actions.}
    \label{fig:toy_stability}
\end{figure}

To confirm (1) and (2) listed above we conduct two experiments. First, to test the action-stability of an algorithm with a neural network model, we train 20 different policies on the same dataset of contexts and rewards, but with resampled actions. Formally, we sample $ \{x_i, r_i\}_{i=1}^N$ from the Gaussian distributions described above and then sample $ a_i \sim \beta(\cdot|x_i)$ with 20 different seeds. We then train each algorithm to convergence and compare the resulting policies by total variation (TV) distance. Results are shown in Figure \ref{fig:toy_stability}. We find that our results from Section \ref{sec:stable} are confirmed: policy-based algorithms are unstable leading to high TV distance between policies trained on different seeds while value-based algorithms are stable.

\begin{figure}
    \centering
    \includegraphics[width=0.3\textwidth]{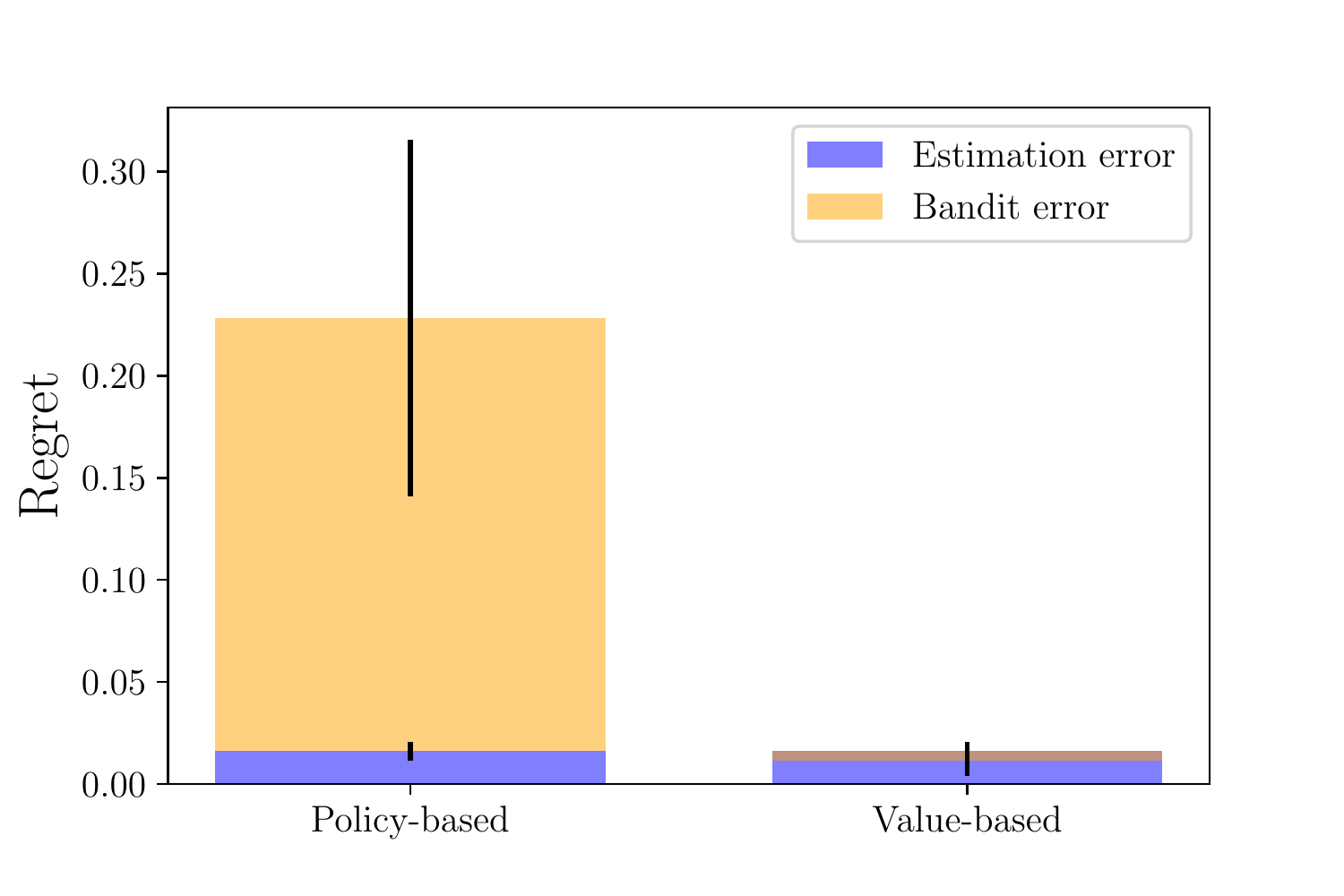}
    \caption{Estimated bandit error by averaging the values calculated on the held-out test sets for 50 independently sampled datasets. Error bars show one standard deviation. While policy-based learning has high bandit error, value-based learning has essentially zero bandit error.}
    \label{fig:toy_regret}
\end{figure}

Second, we estimate the bandit error of each algorithm. To do this we train policies to convergence for the policy-based, value-based, and full-feedback objectives 50 independently sampled datasets (where now we also resample the contexts and rewards). For this estimate, we assume perfect optimization and no approximation error. Each estimate is calculated on a held out test set. Explicitly, let $ \pi_B^j, \pi_Q^j, \pi_F^j$ are the policy-based, value-based, and full-feedback policies trained on dataset $ S^j$ with seed $ j$ and corresponding test set $ T^j$. Then we estimate bandit error as $ \frac{1}{J}\sum_{j=1}^J V(\pi_F^j;T^j) - V(\pi_B^j; T^j)$. Similarly, since we know $ \theta$ we can compute $ \pi^*$ and use this to estimate the estimation error. The results shown in Figure \ref{fig:toy_regret} demonstrate that the policy-based algorithm suffers from substantially more bandit error and thus more regret.


\subsection{Classification data}

Most prior work on offline contextual bandits conducts experiments on classification datasets that are transformed into bandit problems \cite{beygelzimer2009offset, dudik2011doubly, swaminathan2015counterfactual, swaminathan2015self, joachims2018deep, chen2019surrogate}. This methodology obscures issues of action-stability because the very particular reward function used (namely 1 for a correct label and 0 for incorrect) makes the policy-based objective action-stable. However, even minor changes to this reward function can dramatically change the performance of policy-based algorithms by rendering the objective action-unstable.

To make a clear comparison to prior work that uses deep neural networks for offline contextual bandits like \citet{joachims2018deep}, we will consider the same image based bandit problem that they do in their work. Namely, we will use the a bandit version of CIFAR-10 \citep{Krizhevsky09learningmultiple}.

To turn CIFAR into an offline bandit problem we view each possible label as an action and assign reward of 1 for a correct label/action and 0 for an incorrect label/action. We use two different behavior policies to generate training data: (1) a uniformly random behavior policy and (2) the hand-crafted policy used in \citep{joachims2018deep}. We train Resnet-18 \citep{he2016deep} models using Pytorch \citep{paszke2019pytorch}. Again full details about the training procedure are in Appendix \ref{app:experiments}.

As explained in Section \ref{sec:stable}, the policy-based objective is stable if and only if the sign of the reward minus baseline is an indicator of the optimal action. To test this insight we consider two variants of policy-based learning: ``unstable'' where we use a baseline of -0.1 so that the effective rewards are 1.1 for a correct label and 0.1 for incorrect and ``stable'' where we use a baseline of 0.1 so that the effective rewards are of 0.9 and -0.1 to make the objective stable\footnote{This corresponds to the optimal value of $ \lambda$ in the experiments of \citet{joachims2018deep}. Our ``stable'' model slightly outperforms theirs, likely due to a slightly better implementation.}. Note that this ``stable'' variant of the algorithm \emph{only} exists because we are considering a classification problem. In settings with more rich structure in the rewards, defining such an algorithm is not possible and only versions of the unstable algorithm would exist.

\begin{figure}
    \centering
    \includegraphics[width=0.48\textwidth]{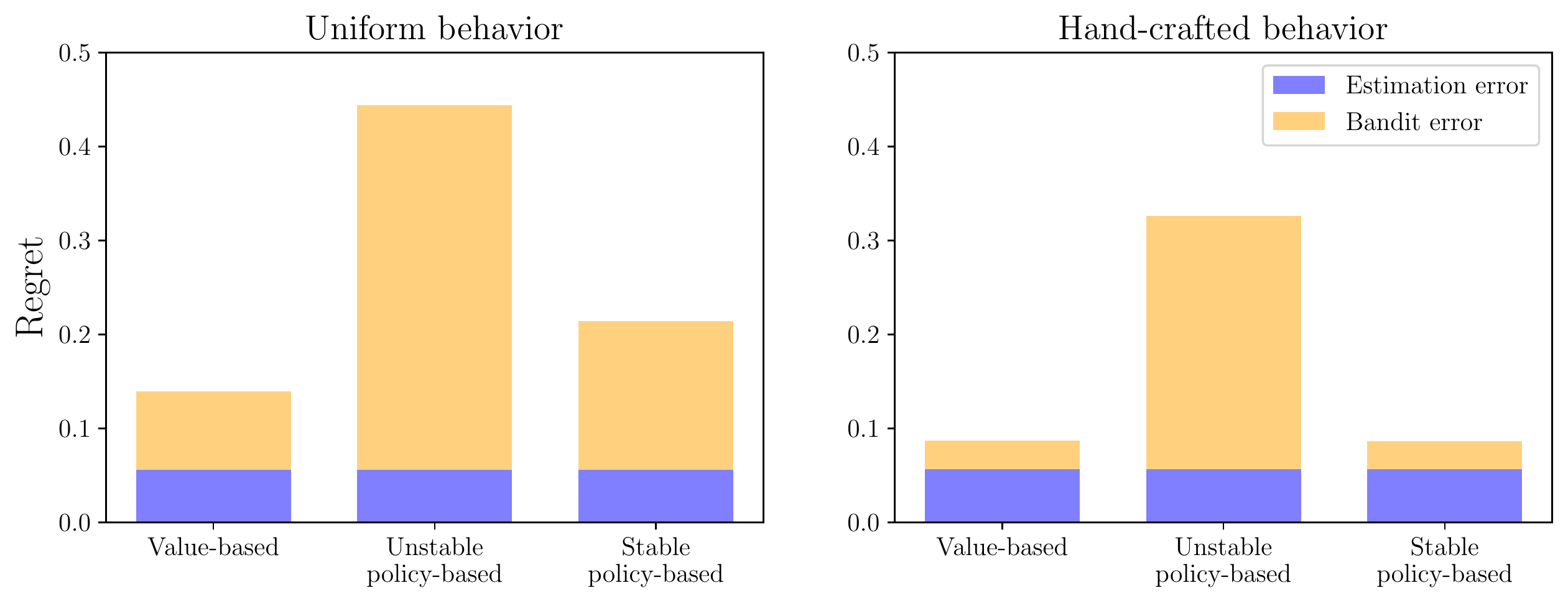}
    \caption{Estimated regret decomposition on CIFAR with uniform behavior (left) and the hand-crafted behavior of \citet{joachims2018deep} (right). We see that the value-based learning has lowest bandit error and unstable policy-based learning the most. On the hand-crafted dataset the stable policy-based algorithm performs as well as value-based learning.}
    \label{fig:cifar_regret}
\end{figure}

We again estimate the regret decomposition as we did with the synthetic data. The difference is that this time we only use one seed since we only have one CIFAR-10 dataset. The results in Figure \ref{fig:cifar_regret} confirm the results from the synthetic data. The standard (unstable) policy-based algorithm suffers from large bandit error. The value-based algorithm has the best performance across both datasets although the ``stable'' policy-based algorithm performs about as well for the hand-crafted behavior policy.

\section{Related work}\label{sec:related}

Now that we have presented our results, we will contrast them with some related work to clarify our contribution.

\subsection{Relation to propensity overfitting}
\citet{swaminathan2015self} introduce what they call ``propensity overfitting''. By providing an example, they show that policy-based algorithms overfit to maximize the sum of propensities $( \sum_i \frac{\pi(a_i|x_i)}{p_i})$ rather than the value when the rewards are strictly positive. They provide a motivating example, but no formal definition of propensity overfitting and argue that it helps to describe the gap between supervised learning and bandit learning. In contrast, we introduce and formally define bandit error, which makes this gap between supervised learning and bandit learning precise and does not rely on the specific algorithm being used. Then we introduce and formally define action-instability, which explains an important cause of bandit error for policy-based algorithms when using large models. 
By mathematically formalizing these ideas we provide a more rigorous foundation for future work.

\subsection{Relation to \cite{joachims2018deep}}

The main related work that considers offline contextual bandits with large neural network models is \citet{joachims2018deep}. Specifically, that paper proposes a policy-based algorithm with an objective of the form: $\frac{r_i(a_i) - \lambda}{\beta(a_i|x_i)} \pi(a_i|x_i)$ for some constant baseline $ \lambda$ determined by a hyperparameter sweep, but motivated by a connection to self-normalized importance sampling.

Our work contrasts with this prior work in two key ways. First, we show that the algorithm proposed in \citet{joachims2018deep} is action-unstable. Specifically, our Lemma \ref{lem:pb-stable} shows that any policy-based algorithm with an objective of the form $ \sum_i f(z_i(a_i)) \pi(a_i|x_i)$ cannot be action-stable unless the sign of $ f(z(a))$ is the indicator of the optimal action. However, since that paper only tests the algorithm on classification problems where the rewards are in $ \{0,1\}$, any setting of $ \lambda$ between 0 and 1 causes the sign of $ f $ to indicate the optimal action. The action-stability analysis shows how this algorithm will struggle beyond the classification setting. 

Second, we show that value-based methods provably and empirically work in the overparameterized setting. In contrast, \citet{joachims2018deep} does not consider value-based methods. We show that value-based methods are not affected by action-stability issues (Lemma \ref{lem:vb-stable}) and have vanishing bandit error (Theorem \ref{thm:reduction}). We empirically test this conclusion on the same CIFAR-10 bandit problem as \citet{joachims2018deep} and find that a value-based approach outperforms the policy-based approach proposed in that paper (Figure \ref{fig:cifar_regret}).

\subsection{Variance of importance weighting}

The importance weighted value estimates used by policy-based algorithms suffer from high variance due to low probability actions that have very large importance weights.
Much prior work focuses on reducing this variance \cite{strehl2010learning, bottou2013counterfactual, swaminathan2015counterfactual}. 
In contrast, the issue we consider, action-instability in the overparameterized setting, is distinct from this variance issue. 
When the policy class is flexible enough to optimize the objective at each datapoint, the optimal predictor in that class does not depend on the importance weights.
Meanwhile action-unstable objectives translate stochasticity in the behavior policy into noise in the objective, causing the overfitting issues that we see in policy-based algorithms. 
In fact, our Theorem \ref{thm:vs} suggests that regret will be worse for more uniform behavior policies when using an action-unstable objective, even though these may be beneficial in terms of variance.
This is born out in our experiments where the behavior is usually \emph{uniform} and \emph{known}, which is a favorable setup in terms of the variance of the value estimates, but an unfavorable setup for action-unstable policy learning algorithms.




\section{Discussion}
We have examined the offline contextual bandit problem with overparameterized models. We introduced a new regret decomposition to separate the effects of estimation error and bandit error. We showed that policy-based algorithms are not action-stable and thus suffer from high bandit error with stochastic behavior policies. 
This is borne out both in the theory and experiments.

It is important to emphasize that our results may not apply beyond the setting we consider in this paper. When there is no strict positivity, there is unobserved confounding, there are continuous actions, or the model classes are small and misspecified then policy-based learning may have lower regret and lower bandit error than value-based learning. 

In future work we hope to extend the ideas from the bandit setting to the full RL problem with longer horizon that requires temporal credit assignment. We predict that action-stability and bandit error remain significant issues there. 
We also hope to investigate action-stable algorithms beyond the simple value-based algorithms we consider here.

\subsection*{Acknowledgements}
We would like to thank Aahlad Puli for thoughtful conversations and Aaron Zweig, Min Jae Song, and Evgenii Nikishin for comments on earlier drafts.

This work is partially supported by the Alfred P. Sloan Foundation, NSF RI-1816753, NSF CAREER CIF 1845360, NSF CHS-1901091, Samsung Electronics, and the Institute for Advanced Study.
DB is supported by the Department of Defense (DoD) through the National Defense Science \& Engineering Graduate Fellowship (NDSEG) Program.

\bibliography{rl}
\bibliographystyle{icml2021}

\newpage

\onecolumn
\appendix
\noindent\rule{\textwidth}{1pt}
\vskip 0.1in
\vskip -\parskip
 \begin{center}
{\textbf{\LARGE Appendix}} \end{center}
\noindent\rule{\textwidth}{1pt}

\section{Action-stability}\label{app:stable}

\vbstable*

\begin{proof}
Consider a datapoint $ z= (x,r)$ which becomes $ z(a) = (x, a, r(a))$ when we sample action $ a$ from the behavior. At this datapoint, the value-based objective for an estimated Q function $ \hat Q$ is 
\begin{align}
    \ell(z(a), \hat Q(x, a)) = (r(a) - \hat Q(x,a))^2
\end{align}
This is minimized at all $ a $ by $\hat Q(x,a) = r(a)$. So setting $ y^* = \hat Q(x, \cdot) = r$, we can exactly minimize $ \ell$ at $ z$. Since such a $ y^*$ exists, the objective is by definition action-stable.
\end{proof}

\pbstable*

\begin{proof}
Consider a datapoint $ z= (x,r)$ which becomes $ z(a) = (x, a, r(a), p(a))$ when we sample action $ a$ from the behavior with probability $ p(a)$. At this datapoint, a generic policy-based objective evaluated on a policy $ \hat \pi $ takes the form
\begin{align}
    \ell(z(a), \hat \pi(a|x)) = f(z(a)) \hat \pi(a|x)
\end{align}
As special examples of the function $ f $ we have the generic policy-based objective from Equation (\ref{eq:pi}) when $ f(z(a)) = \frac{r(a)}{p(a)}$. Moreover we can incorporate any baseline function $ b(x)$ so that $ f(z(a)) = \frac{r(a) - b(x)}{p(a)}$. This algorithm covers the one presented by \citet{joachims2018deep}.

Now to prove the claim, we have three cases: (1) $ f(z(a)) < 0$ for all $ a$, (2) $ f(z(a)) > 0 $ for at least two actions $ a_1, a_2$, and (3) $ f(z(a)) > 0$ at exactly one action $ a_1$. We will show that in cases 1 and 2 the objective is action-unstable, but in case 3 it is action-stable.

\paragraph{Case 1.} Assume that $ f(z(a)) < 0$ for all $ a$. Now for any given $ a $ to maximize the objective $f(z(a))\hat \pi(a|x)$ while ensuring that $ \hat \pi(a|x) $ is a valid probability we must set $ \hat \pi(a|x) = 0$. But, if we set $ \hat\pi(a|x) = 0$ for all $ a$, we no longer have a valid probability distribution, since $ 0 \not \in \Delta^K$. Thus, we cannot find $ y^* \in \Delta^K$ that optimizes the loss at $ z $ across all actions, so the objective is action-unstable. 

\paragraph{Case 2.} Assume that $ f(z(a)) > 0 $ for at least two actions $ a_1, a_2$. Now at $ a_1, a_2$ the objective $f(z(a))\hat \pi(a|x)$ is maximized by setting $ \pi(a|x) = 1$. However, there is no valid element $ y $ of $ \Delta^K$ such that $ y(a_1) = 1$ and $y(a_2) = 1$. Thus, we cannot find $ y^* \in \Delta^K$ that optimizes the loss at $ z $ across all actions, so the objective is action-unstable. 

\paragraph{Case 3.} Assume $f(z(a)) > 0$ at exactly one action $ a_1$. Then at action $ a_1$ we can maximize $ f(z(a_1))\hat \pi(a_1|x)$ by setting $ \hat\pi(a_1|x) = 1$. And since $ f(z(a)) \leq 0$ for all other actions $ a \neq a_1$, we can maximize $ f(z(a))\hat \pi(a|x)$ by setting $ \hat\pi(a|x) = 0$. Now since $ \1[a = a_1] \in \Delta^K$, there does exist a vector $ y^* \in \mathcal{Y}$ which exactly optimizes $ \ell$ regardless of which action is sampled. So, the objective is action-stable if and only if we are in this case.
\end{proof}

\section{Value-based learning}\label{app:value}

\reduction*

\begin{proof}
The proof follows directly from linking the subsequent lemmas with $ \hat \pi = \pi_{\hat Q_{S_B}}$ and $ \Pi$ be the set of all policies in Lemma \ref{lem:mismatch}.
\end{proof}


\begin{restatable}[Mismatch: from MSE to Regret]{lemma}{mismatch}\label{lem:mismatch}
Assume strict positivity. Let $\hat \pi$ be the greedy policy with respect to some $ \hat Q $ and let $ \Pi$ be any class of policies to compete against, which contains $ \hat \pi$. Then
\begin{align}
    \sup_{\pi\in \Pi} V(\pi) - V(\hat\pi) \leq 2 \sqrt{\sup_{\pi\in \Pi} \E_{x,a \sim \mathcal{D}, \pi}[ (Q(x,a) - \hat Q(x,a))^2]}
\end{align}
\end{restatable}

\begin{proof}
We can expand the definition of regret and then add and subtract and apply a few inequalities. Let $ \bar \pi $ be the policy in $ \Pi $ which maximizes $ V$. Then 
\begin{align}
    \sup_{\pi\in \Pi}V(\pi) - V(\hat \pi) &=  \E_x\bigg[\E_{a\sim \bar\pi|x}[Q(x,a)] - \E_{a\sim\hat\pi|x}[Q(x,a)]\bigg]\\
    &= \E_x\bigg[\E_{a\sim \bar\pi|x}[ Q(x,a)] - \E_{a\sim\hat\pi|x}[\hat Q(x,a)] + \E_{a\sim\hat\pi|x}[\hat Q(x,a)] - \E_{a\sim\hat\pi|x}[Q(x,a)]\bigg]\\
    &\leq \E_x\bigg[\E_{a\sim \bar\pi|x}[| Q(x,a) -\hat Q(x,a)|] + \E_{a\sim\hat\pi|x}[|Q(x,a)-\hat Q(x,a)|]\bigg]\\
    &\leq \sqrt{\E_x\E_{a\sim \bar\pi|x}[(Q(x,a) -\hat Q(x,a))^2]} + \sqrt{\E_x\E_{a\sim \hat\pi|x}[(Q(x,a) -\hat Q(x,a))^2]} \\
    &\leq 2  \sqrt{ \sup_{\pi\in \Pi}\E_x\big[\E_{a\sim \pi|x}[(Q(x,a)-  \hat Q(x,a))^2]\big]}
\end{align}
The first inequality holds since $ \hat \pi $ maximizes $ \hat Q$ and by using the definition of absolute value, the second by Jensen, and the third by introducing the supremum.
\end{proof}


\begin{restatable}[Transfer: from $\beta$ to $\pi$]{lemma}{transfer}\label{lem:transfer}
Assume strict positivity and take any Q-function $ \hat Q $ and any policy $ \pi$, then
\begin{align}
    \E_{x, a\sim \mathcal{D}, \pi}[Q(x,a) - \hat Q(x,a))^2] < \frac{1}{\tau}\bigg(\E_{x,a \sim \mathcal{D}, \beta}[(Q(x,a) - \hat Q(x,a))^2]\bigg).
\end{align}
\end{restatable}
\begin{proof} Let $ \pi$ be any policy. Then 
\begin{align}
    \E_{x}\E_{a\sim \pi|x} [(Q(x,a) - \hat Q(x,a))^2] &=  \int_x p(x) \sum_a \pi(a|x) (Q(x,a) - \hat Q(x,a))^2 dx \\ 
    &=  \int_x \sum_a \pi(a|x) \frac{\beta(a|x)}{\beta(a|x)} p(x) (Q(x,a) - \hat Q(x,a))^2 dx \\ 
    &<  \frac{1}{\tau}\int_x \sum_a \beta(a|x) p(x) (Q(x,a) - \hat Q(x,a))^2 dx\\
    &= \frac{1}{\tau} \E_{x,a \sim \mathcal{D}, \beta}[(Q(x,a) - \hat Q(x,a))^2] 
\end{align}
where we use a multiply and divide trick and apply the definition of strict positivity to ensure that $ \frac{\pi(a|x)}{\beta(a|x)} < \frac{1}{\tau}$.
\end{proof}

\section{Policy-based learning}\label{app:pb}

\subsection{In-sample regret}

\begin{lemma}\label{lem:policy_decomp}
Let $ \Pi$ be an interpolating class and $ K = 2$. Then there exists a $ \pi_B $ as defined in Equation (\ref{eq:pi}) such that
\begin{enumerate}
    \item the behavior of $ \pi_B$ at each datapoint $ x_i \in S$ only depends on $ a_i, r_i(a_i)$, and $ p_i$
    \item $ \pi_B(\cdot|x_i) \in \{(0,1),(1,0)\}$.
\end{enumerate} 
\end{lemma}
\begin{proof}
We will begin by proving part 2. Note that the objective that $ \pi_B$ optimizes takes the form $ \frac{r_i(a_i)}{p_i}\pi(a_i|x_i)$ at each datapoint. Since probabilities are constrained to $[0,1]$ this is optimized by $ \pi(a_i|x_i) = 0$ if $ \frac{r_i(a_i)}{p_i} < 0$ and $ \pi(a_i|x_i) = 1$ if $ \frac{r_i(a_i)}{p_i} > 0$. Since we have an overparameterized model class, we know that $ \Pi$ contains a $ \pi_B$ that can exactly choose the optimizer at each datapoint. Since $ K=2$, once we know $ \pi(a_i|x_i)$ we immediately have $ \pi(\bar a_i|x_i) = 1 - \pi(a_i|x_i)$ (where $ \hat a_i$ is the action that is not equal to $ a_i$). Thus $ \pi_B(\cdot|x_i) \in \{(0,1),(1,0)\}$.

Now part 1 follows directly since the above reasoning showed that $ \pi_B(\cdot|x_i)$ is defined precisely by the sign of $ \frac{r_i(a_i)}{p_i}$ and the identity of $ a_i$. 
\end{proof}

\vsthm*

\begin{proof}
By part 1 of Lemma \ref{lem:policy_decomp} and linearity of expectation we can decompose the expected in-sample value as
\begin{align*}
    \E_{S}[V(\pi^*;S) - V(\pi_B;S)] = \frac{1}{N}\sum_{i=1}^N \E_{x_i, r_i, a_i}\bigg[ \E_{a\sim \pi^*}\E_{r|x_i}[r(a)] -  \E_{a\sim \pi_B}\E_{r|x_i}[r(a)]\bigg].
\end{align*}

Since the data are iid we further have that
\begin{align*}
    \E_{S}[V(\pi^*;S) - V(\pi_B;S)] = \E_{x_i, r_i, a_i}\bigg[ \E_{a\sim \pi^*}\E_{r|x_i}[r(a)] -  \E_{a\sim \pi_B}\E_{r|x_i}[r(a)]\bigg].
\end{align*}

Define the event $ U_{x,r}$ to be the event that the policy-based objective is action-unstable at $ x,r$. So $ p_u(x) = \E_{r|x}[\1[U_{x,r}]]$.
We can split this expectation up into stable and unstable parts by conditioning on either $ \overline U_{x_i, r_i}$ or $ U_{x_i, r_i}$, and lower bound the regret on the stable datapoints by 0:
\begin{align*}
    \E_{S}[V(\pi^*;S) - V(\pi_B;S)] &= \E_{x_i, r_i |\overline U_{x_i, r_i}} \E_{a_i|x_i}\bigg[ \E_{a\sim \pi^*}\E_{r|x_i}[r(a)] -  \E_{a\sim \pi_B}\E_{r|x_i}[r(a)]\bigg] \\
    &\qquad + \E_{x_i, r_i |U_{x_i, r_i}} \E_{ a_i|x_i}\bigg[ \E_{a\sim \pi^*}\E_{r|x_i}[r(a)] -  \E_{a\sim \pi_B}\E_{r|x_i}[r(a)]\bigg]\\
    &\geq \E_{x_i, r_i |U_{x_i, r_i}} \E_{ a_i|x_i}\bigg[ \E_{a\sim \pi^*}\E_{r|x_i}[r(a)] -  \E_{a\sim \pi_B}\E_{r|x_i}[r(a)]\bigg].
\end{align*}

By part 2 of Lemma \ref{lem:policy_decomp} we know that $ \pi_B(\cdot|x_i) $ is either $ (1,0)$ or $ (0,1)$. Conditioned on the objective being unstable at $ x_i$ and using the fact that there are only two actions, we know that $ \pi_B(x_i)$ must be different depending on whether $ a_i = 1$ or $ a_i = 2$. Define $ a_{i,B}^1$ to be the action that $ \pi_B$ selects at $ x_i $ when $ a_i = 1$ and $ a_{i,B}^2$ the action when $ a_i = 2$.
Let $ a_i^*$ be the action chosen by the deterministic optimal policy $ \pi^*$ at $ x_i$.
Thus we can split the expectation over $ a_i$ in the above expression and then plug in definitions to get: 
\begin{align*}
    \E_{S}[V(\pi^*;S) - V(\pi_B;S)] &\geq \E_{x_i, r_i |U_{x_i, r_i}} \bigg[\beta(a_i = 1|x_i) \E_{r|x_i}[r(a^*_i) - r(a_{i,B}^1)] + \beta(a_i = 2|x_i) \E_{r|x_i}[r(a^*_i) - r(a_{i,B}^2)]\bigg].
\end{align*}
Since we assumed that $ \beta (a|x_i) \geq \tau$ for all $ a$ we can lower bound the above by
\begin{align*}
    \E_{S}[V(\pi^*;S) - V(\pi_B;S)] &\geq \tau \E_{x_i, r_i} \bigg[\1[U_{x_i, r_i}]\bigg( \E_{r|x_i}[r(a_i^*) - r(a_{i,B}^1)] +  \E_{r|x_i}[r(a_i^*) - r(a_{i,B}^2)]\bigg)\bigg].
\end{align*}
Finally, we note that since $ a_{i,B}^1 \neq a_{i,B}^2$ and there are only 2 actions that the above is precisely
\begin{align*}
    \E_{S}[V(\pi^*;S) - V(\pi_B;S)] &\geq  \tau \E_{x_i, r_i} \bigg[\1[\overline E_{x_i, r_i}] \E_{r|x_i}[r(a_i^*) - r(a \neq a^*_i)]\bigg]\\
    &= \tau \E_{x_i, r_i} [\1[\overline E_{x_i, r_i}] \Delta_r(x_i)]\\
    &= \tau \E_{x_i} [\E_{r_i|x_i}[\1[U_{x_i, r_i}]] \Delta_r(x_i)]\\
    &= \tau \E_{x_i} [p_u(x_i) \Delta_r(x_i)]\\
    &= \tau \E_{x} [p_u(x) \Delta_r(x)].
\end{align*}
\end{proof}

\subsection{Connection to noisy classification}\label{app:noisy}

This section states and proves the Theorem referenced in the main text connecting action-unstable policy-based learning with noisy classification.

\begin{restatable}[Noisy classification reduction]{theorem}{noisy}\label{thm:noisy}
Take any noise level $ \eta < 1/2$ and any binary classification problem $ \mathcal{C}$ consisting of a distribution $\mathcal{D}_\mathcal{C}$ over $ \mathcal{X}$ and a labeling function $ y_{\mathcal{C}}: \mathcal{X} \to \{-1,1\}$.
There exists an offline contextual bandit problem $ \mathcal{B}$ with noiseless rewards such that 
\begin{enumerate}
    \item Maximizing $ \hat V_B$ in $\mathcal{B}$ is equivalent to minimizing the 0/1 loss on a training set drawn from $ \mathcal{C}$ where labels are flipped with probability $ \eta$.
    \item Maximizing $ \hat V_F$ in $\mathcal{B}$ is equivalent to minimizing the 0/1 loss on a training set drawn from $ \mathcal{C}$ with noiseless training labels.
\end{enumerate}
\end{restatable}

\begin{proof}
First we will construct the bandit problem $ \mathcal{B}$ with two actions corresponding to the classification problem $ \mathcal{C}$. 
For any constant $ c_r > 0 $ we define $ \mathcal{B}$ by 
\begin{align}
    x \sim \mathcal{D}_{\mathcal{C}}, \qquad r|x = \begin{cases}c_r(1-\eta, \eta) & y_\mathcal{C}(x) = 1\\ c_r(\eta, 1-\eta) & y_\mathcal{C}(x) = -1\end{cases}, \qquad \beta(1|x) = \begin{cases}1-\eta & y_\mathcal{C}(x) = 1\\ \eta & y_\mathcal{C}(x) = -1\end{cases}
\end{align}
Now we will show that in this problem, $ \hat V_B$ is equivalent to the 0/1 loss for $ \mathcal{C}$ with noisy labels. To do this first note that by construction, for $ x $ with $ y_\mathcal{C}(x) = 1$ we have $ \frac{r(1)|x}{\beta(1|x)} = \frac{c_r(1-\eta)}{1-\eta} = c_r$ and $ \frac{r(2)|x}{\beta(2|x)} = \frac{c_r\eta}{\eta} = c_r$, and similarly for $ x $ with $ y_\mathcal{C}(x) = -1$ we have $ \frac{r(1)|x}{\beta(1|x)} = \frac{c_r\eta}{\eta} = c_r$ and $ \frac{r(2)|x}{\beta(2|x)} = \frac{c_r(1-\eta)}{1-\eta} = c_r$.
\begin{align}
    \hat V_B(\pi) &= \frac{1}{N} \sum_{i=1}^N r_i(a_i) \frac{\pi(a_i|x_i)}{\beta(a_i|x_i)} = \frac{1}{N} \sum_{i=1}^N  \frac{r_i(a_i)}{\beta(a_i|x_i)} \pi(a_i|x_i) \\
    &= \frac{c_r}{N} \sum_{i=1}^N  \pi(a_i|x_i)  
\end{align}
This is equivalent to 0/1 loss with noisy labels since $ \beta$ generates $ a_i$ according to $ y_\mathcal{C}$ where the label is flipped with probability $ \eta$.

Now we will show that $ \hat V_F$ is equivalent to the 0/1 loss for $ \mathcal{C}$ with clean labels. Note that by construction $ r(a)|x = c_r \eta + \pi^*(a|x)c_r (1 - 2 \eta)$. So,
\begin{align}
    \hat V_F(\pi) &= \frac{1}{N} \sum_{i=1}^N \langle r_i, \pi(\cdot|x_i)\rangle = \frac{c_r}{N} \sum_{i=1}^N \langle \eta \textbf{1} + (1 - 2\eta) \pi^*(\cdot|x_i), \pi(\cdot|x_i)\rangle \\
    &= \frac{c_r \eta}{N} +  \frac{c_r (1 - 2\eta)}{N} \sum_{i=1}^N \langle  \pi^*(\cdot|x_i), \pi(\cdot|x_i)\rangle
\end{align}
This is equivalent to 0/1 loss with noisy labels since $ \pi^*$ exactly corresponds to $ y_\mathcal{C}$.
\end{proof}

\subsection{Nearest Neighbor}

\nn*

\begin{proof}
First we need to formally define the nearest neighbor rules that interpolate the objectives $ \hat V_B$ and $ \hat V_F$. These are simple in the case of two actions. Let $ i(x)$ be the index of the nearest neighbor to $ x$ in the dataset. Then
\begin{align}
    \pi_B(a|x) = \begin{cases}1 & \big(a = a_{i(x)} \ \textsc{and}\  r_{i(x)}(a_{i(x)}) > 0\big) \textsc{or} \big(a \neq a_{i(x)} \ \textsc{and}\  r_{i(x)}(a_{i(x)}) \leq 0\big) \\ 0 & otherwise.\end{cases}
\end{align}
This is saying that $ \pi_B$ chooses the same action as the observed nearest neighbor if that reward was positive, and the opposite action if that was negative. 
And for the full feedback we just choose the best action from the nearest datapoint.
\begin{align}
    \pi_F(a|x) = \begin{cases}1 & a = \arg\max_{a'} r_{i(x)}(a')\\ 0 & otherwise. \end{cases}
\end{align}

Now we can construct the problem instances needed for the Theorem.
To construct the example, take a bandit problem with two actions (called 1 and 2):
\begin{align*}
    x \sim U([-1,1]), \quad r|x = (1, 1 + \Delta_r), \quad \beta(1|x) = \beta(2|x) = 1/2 \ \forall\ x,a
\end{align*}
The true optimal policy has $\pi^*(2|x) = 1$ for all $x$ and $ V(\pi^*) = 1+\Delta_r$.
The policy with full feedback $ \pi_F $ is to always choose action 2, since every observation will show that action 2 is better.

Now, we will show that in the limit of infinite data, $ \pi_F$ has no regret. Since the rewards are noiseless, the maximum observed reward at a context is exactly the optimal action at that context. Thus, we precisely have a classification problem with noiseless labels so that the Bayes risk is 0.
Since we $\pi^*$ is continuous, the class conditional densities (determined by the indicator of the argmax of $ Q$) are piecewise continuous.
This allows us to apply the classic result of \cite{cover1967} that a nearest neighbor rule has asymptotic risk less twice the Bayes risk, which in this case is zero. This means that asymptotically $ P(\pi_F(a|x) \neq \pi^*(a|x)) = 0$ which immediately gives the second desired result of zero regret in the limit of infinite data under full feedback.

Now we note that since rewards are always positive, we can simplify the definition of $ \pi_B$ as 
\begin{align}
    \pi_{B}(a|x) = \1[a = a_{i(x)}].
\end{align}

Then we have that
\begin{align}
    V(\pi_F) - V(\pi_{B}) &= \E_x [\E_{a\sim \pi_F|x}[Q(x,a)] - \E_{a\sim \pi_{B}|x}[Q(x,a)] ]\\
    &= \E_x [\Delta_r + 1 - (\pi_{B}(1|x) + \pi_{B}(2|x) (\Delta_r + 1))] ] \\
    &= \Delta_r + 1 - \E_x[\1[a_{i(x)} = 1] + (\Delta_r + 1)\1[a_{i(x)} = 2]]
\end{align}
Taking expectation over $ S$ we get
\begin{align}
    \E_S[V(\pi_F) - V(\pi_{B})] &= \E_S[\Delta_r + 1 - \E_x[\1[a_{i(x)} = 1] + (\Delta_r + 1)\1[a_{i(x)} = 2]]]\\
    &= \Delta_r + 1 - \E_x[P_S(a_{i(x)} = 1) + (\Delta_r + 1)P_S(a_{i(x)} = 2)]]\\
    &= \Delta_r + 1 - \E_x[\frac{1}{2} + (\Delta_r + 1)\frac{1}{2}]]\\
    &= \frac{\Delta_r}{2}
\end{align}
This construction did not depend on the size of the dataset, so it is even true as the number of datapoints tends to infinity.
\end{proof}

\section{Discussion of doubly robust algorithms}\label{app:dr}

Before going into the comparison, we will define the doubly robust algorithm \cite{dudik2011doubly} in our notation. Specifically,
\begin{align}\label{eq:dr}
    \widehat V_{DR} (\pi) := \sum_{i=1}^N \bigg[ \sum_a \pi(a|x_i) \hat Q(x_i, a) + \frac{\pi(a_i|x_i)}{\beta(a_i|x_i)} (r_i(a_i) - \hat Q(x_i, a_i)) \bigg], \qquad \hat \pi_{DR} = \arg\max_{\pi \in \Pi} \widehat V_{DR}(\pi)
\end{align}

As stated in the main text, when we use overparameterized models and train $ \hat Q$ on the same data that we use to optimize the policy, then doubly robust methods are equivalent to the vanilla value-based algorithm. This is formalized in Lemma \ref{lem:dr_equiv} below.

This equivalence can be avoided by using crossfitting so that $ \hat Q$ is not trained on the same data as $ \pi$. However, then it is possible that the doubly robust policy objective becomes action-unstable. This is true \emph{even with} access to the true $Q$ function, but requires stochastic rewards. To construct such an example we leverage the stochastic rewards so that instability only occurs at datapoints where certain reward vectors are sampled. This is shown in Lemma \ref{lem:dr_unstable} below.

One final point is to consider the motivation for doubly robust methods. Usually it is motivated by concerns about consistency of the value function estimation or estimation of behavior policy \cite{dudik2011doubly}. However, in our setting we have (1) an overparamterized model class which is large enough to contain the true value function, and (2) exact access to the behavior probabilities. So it is not clear why doubly robust methods would be motivated in our setting. 

\begin{lemma}[Equivalence of DR and vanilla VB]\label{lem:dr_equiv}
When we use overparameterized models and do not use crossfitting, doubly robust learning from Equation (\ref{eq:dr}) is equivalent to vanilla value-based learning from Equation (\ref{eqn:hatQ}).
\end{lemma}

\begin{proof}
When the model for $ \hat Q$ is overparameterized and trained on the full dataset, we know that $ \hat Q(x_i, a_i) = r_i(a_i)$. Thus we get that 
\begin{align}
    \widehat V_{DR} (\pi) &= \sum_{i=1}^N \bigg[ \sum_a \pi(a|x_i) \hat Q(x_i, a) + \frac{\pi(a_i|x_i)}{\beta(a_i|x_i)} (r_i(a_i) - \hat Q(x_i, a_i)) \bigg]\\
    &= \sum_{i=1}^N \bigg[ \sum_a \pi(a|x_i) \hat Q(x_i, a) + \frac{\pi(a_i|x_i)}{\beta(a_i|x_i)} (0) \bigg]\\
    &= \sum_{i=1}^N \sum_a \pi(a|x_i) \hat Q(x_i, a) 
\end{align}
With an overparameterized policy class, we can exactly recover the greedy policy relative to $ \hat Q$ to optimize this objective.
\end{proof}

\begin{lemma}[Instability of DR]\label{lem:dr_unstable}
There exist problems with stochastic rewards where even with access to the exact Q function, the doubly robust policy objective is action-unstable with probability 1/2. 
\end{lemma}

\begin{proof}
We need only consider one datapoint since the action-stability property is defined on a per datapoint basis. To make this construction we will consider only two actions.
\begin{align}
    r|x = \begin{cases} (0, 1) & w.p.\ 1/2\\ (0, -2) & otherwise\end{cases}, \qquad \beta(\cdot|x) = (1/2, 1/2)
\end{align}
So, we know that
\begin{align}
    Q(\cdot|x) = (0,-0.5)
\end{align}
Now we claim that when the sampled datapoint has $ r = (0,1)$ the doubly robust objective is action-unstable (and this happens with probability 1/2 by construction).
We can explicitly expand the DR objective for the policy $ \pi$ at $ x $ when action $ a $ is sampled
\begin{align}
    \ell_{DR}(\pi, x, a, r) = \pi(1|x) \cdot 0 + \pi(2|x) \cdot (-0.5) + \frac{\pi(a|x)}{1/2}(r(a) - Q(x,a))
\end{align}
So when $a = 1$ we have $ r(a) = 0$ and $ Q(x,a) = 0$ so that
\begin{align}
    \ell_{DR}(\pi, x, a, r) = \pi(2|x) \cdot (-0.5) + 2 \cdot \pi(1|x)(0 - 0) = \pi(2|x) \cdot (-0.5)
\end{align}
And when $ a=2$ we have $ r(a) = 1$ (because that was the sampled reward) and $ Q(x,a) = 0$ so that
\begin{align}
    \ell_{DR}(\pi, x, a, r) = \pi(2|x) \cdot (-0.5)  + 2 \cdot \pi(2|x) (1 - (-0.5)) = \pi(2|x) \cdot (2.5)
\end{align}
Now, this is clearly action-unstable since the optimizer when $ a=1 $ is sampled is $ \pi(\cdot|x) = (1,0)$ while when $ a = 2$ is sampled we get $ \pi(\cdot|x) = (0,1)$. 
\end{proof}

\section{Experiments}\label{app:experiments}

\subsection{Synthetic data}

\paragraph{Data.} As described in the main text we sample some hidden reward matrix $ \theta$ and then sample contexts and rewards from isotropic Gaussians:
\begin{align*}
    \theta \sim U([0,1]^{K\times d}),\quad x \sim \mathcal{N}(0, I_d), \quad r \sim \mathcal{N}(\theta x, \epsilon I_d).
\end{align*}
Actions are sampled according to a uniform behavior:
\begin{align*}
    a\sim \beta(\cdot|x) = U(\{1,\dots,K\}).
\end{align*}
We set $ K=2, d = 10, \epsilon = 0.1$. For each random seed we take $ N = 100$ training points and sample an independent test set of 500 points.
For experiment 1 we sample $ \theta$ and one dataset of $ x,r$ tuples, then we sample 20 independent sets of actions. For experiment 2 we sample all parameters separately to construct each of the 50 datasets. 

\paragraph{Model.} For policies and Q functions we use a multilayer perceptron with one hidden layer of width 512 and ReLU activations. The only difference between policy and Q architecture is that the policy has a softmax layer on the output so that the output is a probability distribution.

\paragraph{Learning.} We train using SGD with momentum. Learning rate is 0.01, momentum is 0.9, batch size is 10, and weight decay is 0.0001. We train every model for 1000 epochs decreasing the learning rate by a factor of 10 after 200 epochs. This trains well past the point of convergence in our experience.

\begin{figure}[h]
    \centering
    \includegraphics[width=0.7\textwidth]{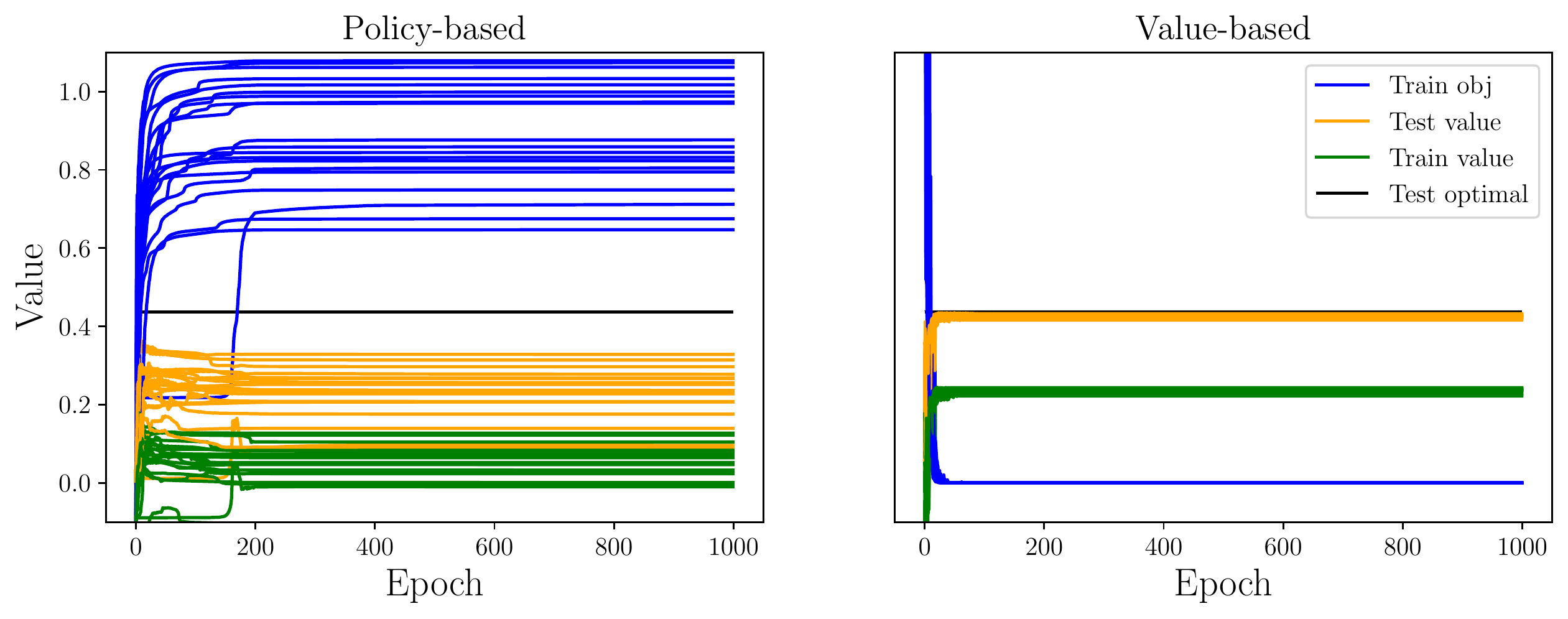}
    \caption{We show learning curves across each of the twenty different action resampled datasets.}
    \label{fig:toy_learning}
\end{figure}

\paragraph{Extended results.}
Figure \ref{fig:toy_learning} shows learning curves for each of the twenty different action datasets from experiment 1. We use ``train obj'' to refer to the training objective which is squared error for value-based learning and $ \hat V_B$ for policy-based learning. We use ``train value'' and ``test value'' to refer to  $ V(\pi;S)$ for $ S$ being the train and test sets respectively. We can evaluate the true value at each datapoint since we know the full reward vector at each datapoint. 

We see that the policy-based objective is dramatically higher than the highest achievable value due to overfitting of the noise in the actions. The gap between train and test value is mot likely explained by noise in the contexts sampled in those respective datasets (by chance the test set has higher value contexts).

\subsection{CIFAR-10}

\paragraph{Data.} We use a bandit version of the CIFAR-10 dataset \citep{Krizhevsky09learningmultiple}. We split the train set into a train set of the first 45000 examples and validation set of the last 5000. We normalize the images and use data augmentation of random flips and crops of size 32. Each of the 10 labels becomes an action. We define rewards to be 1 for a correct prediction and 0 for an incorrect prediction.
We use two different behavior policies. One is a uniform behavior that selects each action with probability 0.1 and the other is the hand-crafted behavior policy from \cite{joachims2018deep}. 

\paragraph{Model.} We use a ResNet-18 \citep{he2016deep} from PyTorch \citep{paszke2019pytorch} for both the policy and the Q function. The only modification we make to accommodate for the smaller images in CIFAR is to remove the first max-pooling layer. 

\paragraph{Learning.} We train using SGD with momentum 0.9,a batch size 128, and weight decay of 0.0001 for 1000 epochs. Training takes about 20 hours for each run on an NVIDIA RTX 2080 Ti GPU. We use a learning rate of 0.1 for the first 200 epochs, 0.01 for the next 200, and 0.001 for the last 600. To improve stability we use gradient clipping and reduce the learning rate in the very first epoch to 0.01.

\paragraph{Extended results.}
Figures \ref{fig:cifar_learning_blbf} and \ref{fig:cifar_learning_uniform} show learning curves for each of the three algorithms we consider across each dataset. The labels refer to the same quantities as they did on the synthetic problem.

One interesting phenomena is that the unstable policy-based algorithm displays a clear overfitting phenomena as we would predict due to the noise in the actions being transferred into noise in the objective. Since we have strictly positive rewards here, this is also an instance of ``propensity overfitting'' \cite{swaminathan2015self}. As a result, limiting the capacity of the model class by early stopping could improve performance somewhat. But by limiting capacity in this way we are exiting the overparameterized/interpolating regime described by \citet{zhang2016understanding}.

\begin{figure}[h]
    \centering
    \includegraphics[width=0.7\textwidth]{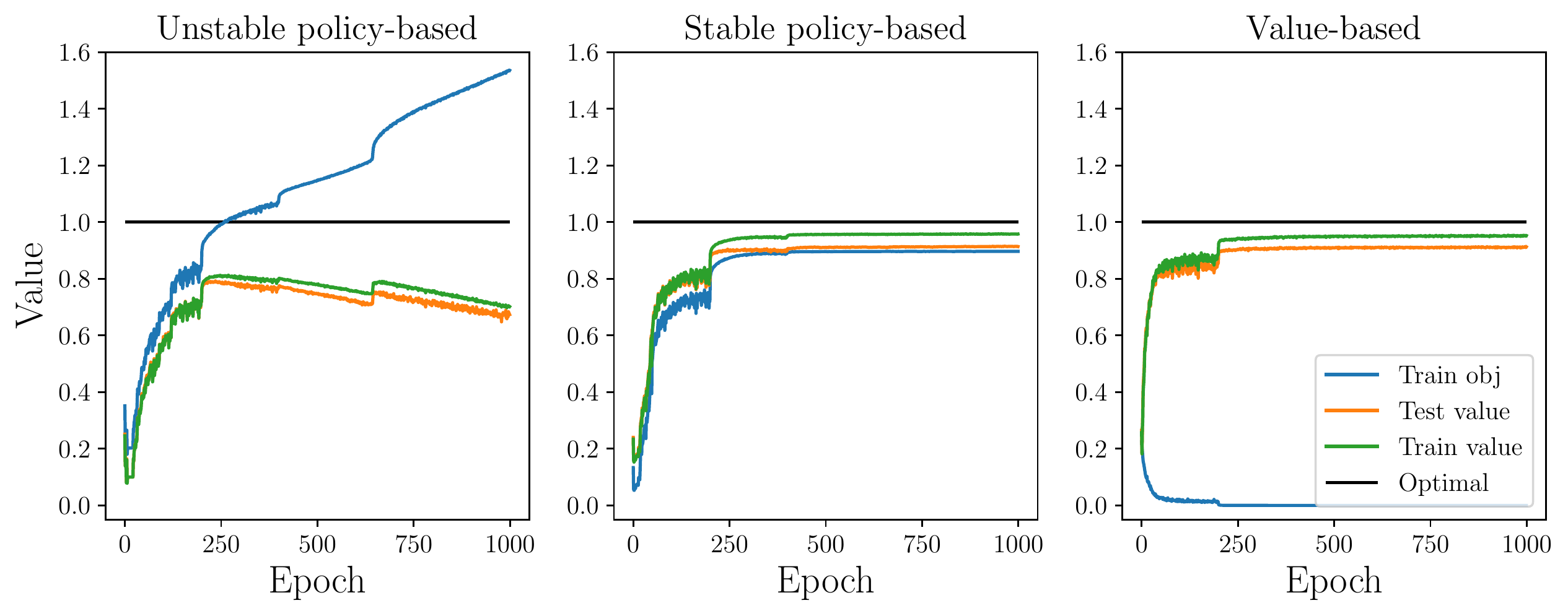}
    \caption{Learning curves on the hand-crafted action dataset.}
    \label{fig:cifar_learning_blbf}
\end{figure}

\begin{figure}[h]
    \centering
    \includegraphics[width=0.7\textwidth]{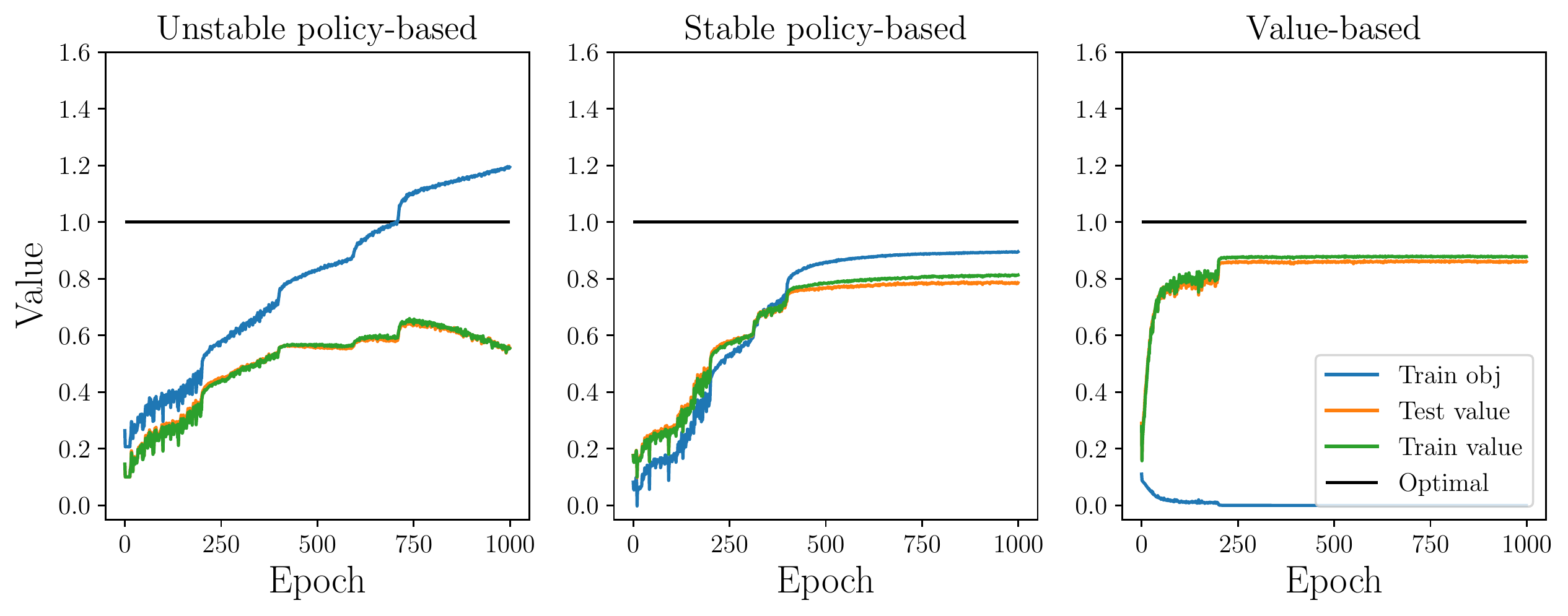}
    \caption{Learning curves on the uniform action dataset.}
    \label{fig:cifar_learning_uniform}
\end{figure}

\section{Small model classes}\label{app:small}

In this section we state and prove theorems that bound each term of our regret decomposition for each algorithm we consider when we use finite model classes. Similar results can be shown for other classical notions of model class complexity. We include these results for completeness, but the main focus of our paper is the overparameterized regime where such bounds are vacuous.

\begin{theorem}[Policy-based learning with a small model class]\label{thm:pol-small}
Assume strict positivity and a finite policy class $ \Pi$. Let $ \varepsilon_\Pi = V(\pi^*) - \sup_{\pi \in \Pi} V(\pi)$. Denote $ \Delta_r = r_{max} - r_{min}$. Then we have that for any $ \delta > 0$ with probability $ 1- \delta$ each of the following holds:
\begin{align*}
    \text{Approximation Error} &= V(\pi^*) - \sup_{\pi\in \Pi}V(\pi) \leq \varepsilon_\Pi\\
    \text{Estimation Error}  &= \sup_{\pi\in \Pi}V(\pi) - V(\pi_F) \leq 2\Delta_r \sqrt{\frac{\log(2|\Pi|/\delta)}{2N}}\\
    \text{Bandit Error}  &= V(\pi_F) - V(\pi_B) \leq \frac{2\Delta_r}{\tau} \sqrt{\frac{\log(2|\Pi|/\delta)}{2N}} 
\end{align*}
\end{theorem}
\begin{proof}
The bound on approximation error follows directly from the definition of $ \varepsilon_\Pi$. The bound on the estimation error follows from a standard application of a Hoeffding bound on the random variables $ X_i = \langle r_i, \pi(\cdot|x_i)\rangle
$ which are bounded by $ \Delta_r$ and a union bound over the policy class. 

The bound on bandit error essentially follows Theorem 3.2 of \cite{strehl2010learning}, we include a proof for completeness:
\begin{align*}
    V(\pi_F) - V(\pi_B) &= V(\pi_F) - \hat V_B(\pi_B) + \hat V_B(\pi_B) - V(\pi_B) \\
    &\leq V(\pi_F) - \hat V_B(\pi_F) + \hat V_B(\pi_B) - V(\pi_B)\\
    &\leq 2 \sup_{\pi \in \Pi} |V(\pi) - \hat V_B(\pi)|\\
    &\leq \frac{2\Delta_r}{\tau} \sqrt{\frac{\log(2|\Pi|/\delta)}{2N}} 
\end{align*}
The first inequality comes from the definition of $ \pi_B$. The second comes since both $ \pi_F, \pi_B \in \Pi$. 
And the last inequality follows from an application of a Hoeffding bound on the random variables $ X_i = r_i(a_i)\frac{\pi(a_i|x_i)}{p_i}$ which are bounded by $ \frac{\Delta_r}{\tau}$ and a union bound over the policy class.
\end{proof}

\begin{theorem}[Value-based learning with a small model class]\label{thm:val-small}
Assume strict positivity and a finite function class $ \mathcal{Q}$ which induces a finite class of greedy policies $ \Pi_\mathcal{Q}$. Let $ \varepsilon_\mathcal{Q} = \inf_{\widehat Q\in \mathcal{Q}}\E_{x,a\sim \mathcal{D}, \beta}[(Q(x,a) - \widehat Q(x,a))^2]$. Denote $ \Delta_r = r_{max} - r_{min}$. Then we have that for any $ \delta > 0$ with probability $ 1- \delta$ each of the following holds:
\begin{align}
    \text{Approximation Error} &= V(\pi^*) - \sup_{\pi\in \Pi_{\mathcal{Q}}}V(\pi) \leq 2\sqrt{\varepsilon_{\mathcal{Q}}/ \tau}\\
    \text{Estimation Error}  &= \sup_{\pi\in \Pi_{\mathcal{Q}}}V(\pi) - V(\pi_F) \leq 2\Delta_r \sqrt{\frac{\log(|\mathcal{Q}|/\delta)}{2N}}\\
    \text{Bandit Error}  &= V(\pi_F) - V(\pi_{\widehat Q}) \leq \frac{10\Delta_r}{\sqrt{\tau}} \sqrt{\frac{\log(|\mathcal{Q}|/\delta)}{N}} + 6\sqrt{\Delta_r} \bigg(\frac{\log(|\mathcal{Q}|/\delta)}{\tau N} \varepsilon_{\mathcal{Q}}\bigg)^{1/4} + 2\sqrt{\varepsilon_{\mathcal{Q}}/ \tau}
\end{align}
\end{theorem}

\begin{proof}
To bound the approximation error, we can let $ \hat \pi $ be the greedy policy associated with a Q-function $ \widehat Q$ and apply Lemmas \ref{lem:mismatch} and \ref{lem:transfer}. This gives us
\begin{align}
     V(\pi^*) - \sup_{\hat \pi\in \Pi_{\mathcal{Q}}}V(\hat \pi) =  \inf_{\widehat Q\in \mathcal{Q}} [V(\pi^*) - V(\hat \pi)] \leq \inf_{\widehat Q\in \mathcal{Q}} \frac{2}{\sqrt{\tau}}\sqrt{\E_{x,a\sim \mathcal{D}, \beta}[(Q(x,a) - \widehat Q(x,a))^2]} = 2 \sqrt{\varepsilon_\mathcal{Q} / \tau}.
\end{align}
The bound on the estimation error follows the same as before from standard uniform convergence arguments.

The bound on the bandit error follows by again applying Lemmas \ref{lem:mismatch} and \ref{lem:transfer} and then making the concentration argument from Lemma 16 of \cite{chen2019information}. Explicitly, our Lemmas give us
\begin{align}
    V(\pi_F) - V(\pi_{\widehat Q}) \leq V(\pi^*) - V(\pi_{\widehat Q}) &\leq \frac{2}{\sqrt{\tau}} \sqrt{\E_{x,a\sim \mathcal{D}, \beta}[(Q(x,a) - \widehat Q(x,a))^2]}.
\end{align}
Then, to bound the squared error term, we can add and subtract:
\begin{align}
    \E_{x,a\sim \mathcal{D}, \beta}[(Q(x,a) - \widehat Q(x,a))^2] &= \E_{x,a\sim \mathcal{D}, \beta}[(Q(x,a) - \widehat Q(x,a))^2] - \inf_{\bar Q \in \mathcal{Q}}\E_{x,a\sim \mathcal{D}, \beta}[(Q(x,a) - \bar Q(x,a))^2] \\ &\qquad+ \inf_{\bar Q \in \mathcal{Q}}\E_{x,a\sim \mathcal{D}, \beta}[(Q(x,a) - \bar Q(x,a))^2]\\
    &\leq \E_{x,a\sim \mathcal{D}, \beta}[(Q(x,a) - \widehat Q(x,a))^2] - \inf_{\bar Q \in \mathcal{Q}}\E_{x,a\sim \mathcal{D}, \beta}[(Q(x,a) - \bar Q(x,a))^2] \\&\qquad+ \varepsilon_\mathcal{Q}.
\end{align}
Now we want to show that the difference in squared error terms concentrates for large $ N$. This is precisely what Lemma 16 from \cite{chen2019information} does using a one-sided Bernstein inequality. This gives us for any $ \delta > 0$ an upper bound with probability $ 1-\delta$ of
\begin{align}
    \frac{56\Delta_r^2 \log(|\mathcal{Q}|/\delta)}{3N} + \sqrt{\varepsilon_\mathcal{Q} \frac{32\Delta_r^2 \log(|\mathcal{Q}|/\delta)}{N}}
\end{align}
Plugging this in and simplifying the constants gives the result. 
\end{proof}

\end{document}